\documentclass[journal]{IEEEtran}
\IEEEoverridecommandlockouts

\usepackage[safe]{tipa}

\usepackage{subfigure}

\usepackage{graphicx}
\usepackage{amsfonts,mathtools}
\usepackage{amsmath}
\usepackage{amsthm}
\newtheorem{theorem}{Theorem}[section]
\newtheorem{corollary}{Corollary}[theorem]
\newtheorem{lemma}[theorem]{Lemma}
\usepackage{longtable,tabularx}

\DeclarePairedDelimiter\ceil{\lceil}{\rceil}

\usepackage{graphicx}
\usepackage{cleveref}

\usepackage[table,dvipsnames]{xcolor}

\usepackage[listings,skins,breakable]{tcolorbox}
{%
\end{tcolorbox}\endgroup}
\newtcbox{\popovnotesmall}{breakable,enhanced jigsaw,nobeforeafter,tcbox raise base,boxrule=0.4pt,top=0mm,bottom=0mm,
  right=0mm,left=4mm,arc=1pt,boxsep=2pt,before upper={\vphantom{dlg}},
  colframe=Fuchsia!75!black,coltext=black,colback=Fuchsia!20,
  overlay={\begin{tcbclipinterior}\fill[Fuchsia!75!black] (frame.south west)
    rectangle node[text=white,font=\sffamily\bfseries\tiny,rotate=90] {AAP} ([xshift=4mm]frame.north west);\end{tcbclipinterior}}}
\MakeRobust\popovnotesmall
\ifdefined
\fi

\def\*#1{\boldsymbol{\mathbf{#1}}}
\def\##1{\bm{\mathsf{#1}}}

\DeclareMathOperator{\tr}{tr}

\usepackage{tikz}
\usetikzlibrary{calc}
\usetikzlibrary{math}

\DeclarePairedDelimiterX{\norm}[1]{\lVert}{\rVert}{#1}

\usepackage{pgfplots}
\pgfplotsset{compat=1.18} 
\usepackage{pgfplotstable}
\usepackage{cbcolors}

\begin{document}

\pgfplotsset{clean/.style={axis lines*=left,
        axis on top=true,
        axis x line shift=0.0em,
        axis y line shift=0.75em,
        every tick/.style={black, thick},
        axis line style = ultra thick,
        tick align=outside,
        clip=false,
        major tick length=4pt}}

\title{The Ensemble Epanechnikov Mixture Filter}

\author{Andrey~A.~Popov,
        Renato~Zanetti,~\IEEEmembership{Senior Member,~IEEE}
\thanks{A. A. Popov is with the Oden Institute for Computational Engineering \& Sciences, the University of Texas at Austin, Austin,
TX, 78712 USA e-mail: apopov@hawaii.edu}
\thanks{R. Zanetti is with the Dept of Aerospace Engineering \& Engineering Mechanics, the University of Texas at Austin, Austin,
TX, 78712 USA e-mail: renato@utexas.edu}
\thanks{Manuscript received \today.}}

\markboth{arXiv}%
{Popov \MakeLowercase{\textit{et al.}}: The Ensemble Epanechnikov Mixture Filter}


\maketitle

\begin{abstract}
    In the high-dimensional setting, Gaussian mixture kernel density estimates become increasingly suboptimal. 
    In this work we aim to show that it is practical to instead use the optimal  multivariate Epanechnikov kernel. 
    We make use of this optimal Epanechnikov mixture kernel density estimate for the sequential filtering scenario through what we term the ensemble Epanechnikov mixture filter (EnEMF).
    We provide a practical implementation of the EnEMF that is as cost efficient as the comparable ensemble Gaussian mixture filter. We show on a static example that the EnEMF is robust to growth in dimension, and also that the EnEMF has a significant reduction in error per particle on the 40-variable Lorenz '96 system.
\end{abstract}


\IEEEpeerreviewmaketitle

\section{Introduction}

State estimation~\cite{bar2004estimation} and data assimilation~\cite{asch2016data, reich2015probabilistic} methods restricted to propagating a single mean and a single covariance are limited to dealing with near-linear near-Gaussian scenarios. 
When dealing with highly non-linear dynamics and measurements, these methods do not have robust convergence guarantees and, in the worst case, actively fight against the goal of estimating the uncertainty of the dynamics of interest.
Non-linear filters~\cite{bar2004estimation} can alleviate some of these problems by incorporate measurement information in a much more general (more than affine) fashion.

The Gaussian sum filter (GSF)~\cite{sorenson1971recursive, terejanu2008uncertainty} is capable of fully representing almost all useful probability density functions through the use of Gaussian mixture models (GMM) and the Gaussian sum update. 
Fundamentally, it suffers from requiring a near-Gaussian update of the covariances and requires sophisticated splitting and merging techniques~\cite{faubel2009split} in order to avoid weight collapse in the components.
This means that the GSF has the potential to not behave well in the highly non-linear setting.

On the other hand, particle filters~\cite{gustafsson2002particle} are capable of representing any probability density through a collection of samples. A sample-based representation of a probability density is simpler than a GMM, though requires significantly more particles to represent many probability densities, particularly in the high-dimensional setting.
The dominant roadblock to the use of particle filters is that they require significant effort to avoid weight collapse or filter collapse by employing either sophisticated resampling or particle flow techniques~\cite{daum2016gromov, van2019particle}.

A filter that bridges the gap between the two methodologies is the ensemble Gaussian mixture filter (EnGMF)~\cite{anderson1999monte, liu2016efficient, popov2022adaptive,popov2023elengmf2}.
The EnGMF relies on a kernel density estimate (KDE) to build a GMM  representation of our prior uncertainty from a particle-based propagation step, incorporates measurement information into the posterior through the use of the Gaussian sum update. 
GMMs are attractive for building KDEs  state estimation because of  well-known ``nice'' properties of the Gaussian distribution~\cite{patel1996handbook}.

In the high-dimensional setting, however, GMMs for KDE become progressively less `efficient' as the dimension grows. 
Kernel density estimates based on the Epanechnikov (\textipa{[j\textschwa p\textturna n\textsuperscript{j}'e\texttoptiebar{t\textctc}n\textsuperscript{j}\textsci k\textturna f]}) kernel minimize the error with respect to the underlying distribution of the particles.
This work presents an ensemble Epanechnikov mixture filter (EnEMF) that takes advantage of the efficiency of the Epanechnikov kernel in order to perform high-dimensional particle filtering.

This work is an expanded version of the conference paper~\cite{popov2024fusion} which explored the Epanechnikov kernel in ensemble mixture model filtering. This work formalizes the correctness of the resampling procedure, expands the filter to work with a more sophisticated Gaussian sum update, adds a new weighting scheme based on the unscented transform, and provides a new numerical example that illustrates the robustness of the approach to increase in dimension.

This paper is organized as follows: \cref{sec:background} provides background on ensemble mixture model filtering,  kernel density estimation, and motivation as to the superiority of the Epanechnikov kernel.
The EnEMF is described in~\cref{sec:EnEMF}, with a practical implementation that makes use of the Gaussian sum update. Numerical experiments on the $n$-dimensional banana problem, and on the 40-variable Lorenz '96 system are provided in~\cref{sec:numerical-experiment}.
Finally \cref{sec:conclusions} provides closing remarks about why we believe the Epanechnikov kernel is suitable for ensemble mixture model filtering.  

\section{Background and Motivation}
\label{sec:background}

\begin{figure}[t]
    \centering
    \begin{tikzpicture}

    \begin{scope}
        \filldraw [fill=tolpurple!5, draw=tolpurple!45,rounded corners, very thick, dashed] (-1.2,-1) rectangle (1,0.4);
        \node[circle, draw=tolblue, fill=tolblue, inner sep=2pt, outer sep=4pt] (A) at (0,0) {};
        \node[circle, draw=tolblue, fill=tolblue, inner sep=2pt, outer sep=4pt] (B) at (-0.75,-0.5) {};
        \node[circle, draw=tolblue, fill=tolblue, inner sep=2pt, outer sep=4pt] (C) at (0.5,-0.25) {};
    \end{scope}

    \begin{scope}[shift={(3,2)}]
        \filldraw [fill=tolpurple!5, draw=tolpurple!45,rounded corners, very thick, dashed] (-1.0,-0.95) rectangle (1.2,1.1);
        \node[circle, draw=tolblue, fill=tolblue, inner sep=2pt, outer sep=4pt] (A2) at (0,0.7) {};
        \node[circle, draw=tolblue, fill=tolblue, inner sep=2pt, outer sep=4pt] (B2) at (0.75,-0.5) {};
        \node[circle, draw=tolblue, fill=tolblue, inner sep=2pt, outer sep=4pt] (C2) at (-0.5,-0.25) {};
    \end{scope}

    \draw[->, color=tolred, very thick] (A) to[out=86,in=170] (A2);
    \draw[->, color=tolred, very thick] (B) to[out=86,in=120] (B2);
    \draw[->, color=tolred, very thick] (C) to[out=86,in=180] (C2);

    \node[rotate=40] at (1.1,1.6) {Propagate};

    \begin{scope}[shift={(5,-1)}]
        \filldraw [fill=tolpurple!5, draw=tolpurple!45,rounded corners, very thick, dashed] (-1.3,-1.25) rectangle (1.5,1.5);
       \node[circle, draw=tolblue, fill=tolblue, inner sep=2pt, outer sep=4pt] (A3) at (0,0.7) {};
        \node[circle, draw=tolblue, fill=tolblue, inner sep=2pt, outer sep=4pt] (B3) at (0.75,-0.5) {};
        \node[circle, draw=tolblue, fill=tolblue, inner sep=2pt, outer sep=4pt] (C3) at (-0.5,-0.25) {};

        \draw[fill=none, color=tolblue, dashed, very thick](0.0,0.7) ellipse (15pt and 13pt);
        \draw[fill=none, color=tolblue, dashed, very thick](0.75,-0.5) ellipse (15pt and 13pt);
        \draw[fill=none, color=tolblue, dashed, very thick](-0.5,-0.25) ellipse (15pt and 13pt);
    \end{scope}

    \draw[->, color=tolred, very thick] (4.3, 2.25) to[out=0,in=90] (5.5, 0.65);

    \node[rotate=-40] at (5.35,2.1) {Build KDE};

    \begin{scope}[shift={(2,-3)}]
        \filldraw [fill=tolpurple!5, draw=tolpurple!45,rounded corners, very thick, dashed] (-1.1,-0.95) rectangle (1.3,1.15);
       \node[circle, draw=tolblue, fill=tolblue, inner sep=2pt, outer sep=4pt] (A4) at (0.2,0.45) {};
        \node[circle, draw=tolblue, fill=tolblue, inner sep=2pt, outer sep=4pt] (B4) at (0.75,-0.1) {};
        \node[circle, draw=tolblue, fill=tolblue, inner sep=2pt, outer sep=4pt] (C4) at (-0.5,-0.25) {};

        \draw[fill=none, color=tolblue, dashed, very thick, rotate around={-35:(0.2,0.45)}](0.2,0.45) ellipse (7pt and 12pt);
        \draw[fill=none, color=tolblue, dashed, very thick, rotate around={10:(0.75,-0.1)}](0.75,-0.1) ellipse (9pt and 14pt);
        \draw[fill=none, color=tolblue, dashed, very thick, rotate around={-10:(-0.5,-0.25)}](-0.5,-0.25) ellipse (9pt and 14pt);

        \begin{scope}[shift={(0.25,0.6)}]
        \node[circle, draw=none, fill=none, inner sep=0pt, outer sep=0pt] (A5) at (0,0) {};
        \node[circle, draw=none, fill=none, inner sep=0pt, outer sep=0pt] (B5) at (-0.75,-0.5) {};
        \node[circle, draw=none, fill=none, inner sep=0pt, outer sep=0pt] (C5) at (0.5,-0.25) {};
        \end{scope}
    \end{scope}

    \draw[->, color=tolred, very thick] (4.75, -2.35) to[out=270,in=0] (3.45, -3);
    \node[rotate=40] at (4.7,-3.1) {Update};

    \draw[->, color=tolred, very thick] (A5) to[out=150,in=270] (C);
    \draw[->, color=tolred, very thick] (B5) to[out=150,in=270] (B);
    \draw[->, color=tolred, very thick] (C5) to[out=100,in=270] (A);

    \node[rotate=-40] at (0.15,-1.7) {Resample};

\end{tikzpicture}
    \caption{Ensemble mixture model filtering diagram. Clock wise from the left-most rectangle: given a collection of particles from the previous time, (i) propagate to the current time,  (ii) build a mixture model representing the prior from the particles through kernel density estimation, (iii) update the mixture model by making use of the measurements to create a mixture model representing the posterior, and (iv) resample to create a new collection of particles. This process is repeated until a desired time is reached or ad infinitum.}
    \label{fig:ensemble-mixture-model-filtering}
\end{figure}
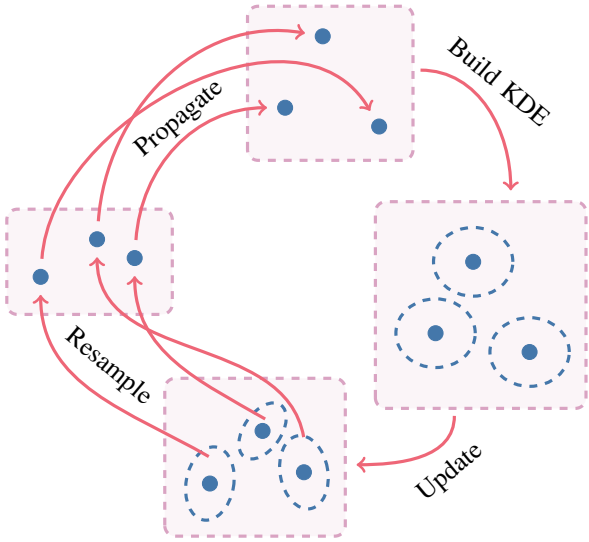

Ensemble mixture modeling filtering, visually described in~\cref{fig:ensemble-mixture-model-filtering}, combines four techniques for state estimation: particle propagation, kernel density estimation, the mixture model update, and resampling. Each one providing an essential component that allows for accurate representations of the ``true''~\cite{jaynes2003probability} posterior with a finite number of samples.

Particle propagation~\cite{reich2015probabilistic} allows for long-term propagation of our knowledge about the state of a particular system through highly non-linear dynamics with arbitrary precision.
Formally, assume that we are given a collection of $N$ independently and identically distributed samples,
\begin{equation}\label{eq:ensemble}
    \*X_{k-1} = \left[x_1, x_2, \dots x_N\right],\quad \*X_{k-1}\in\mathbb{R}^{n\times N},
\end{equation}
from the $n$-dimensional probability distribution of interest at time index $k-1$. 
With abuse of notation, we can write their propagation through some (discrete or continuous, stochastic or deterministic) dynamics,
\begin{equation}
    \*X_k = F(\*X_{k-1}),
\end{equation}
resulting in a collection of $N$ particles at time index $k$. Generally the dynamics $F$ can incorporate some measure of uncertainty through some type of process noise or measure of model error, though for simplicity and without any real loss of generality we ignore any such notions for the remainder of this work.

The resulting particles can be used to represent our knowledge about the state of a system through the use of sample statistic, and, for the purposes of this work, can be used to construct a representation of the distribution of interest at time index $k$ through the use of kernel density estimation techniques~\cite{silverman2018density}.
If the distribution of interest is denoted by $f$,  the kernel density estimate of the distribution of interest  described by the ensemble~\cref{eq:ensemble} is the mixture model,
\begin{equation}\label{eq:KDE-estimate}
    \widetilde{f}_{\*X_N}(x) = \frac{1}{N} \sum_{i=1}^N \mathcal{K}(x\,;\,x_i, \*C_i),
\end{equation}
where $\mathcal{K}$ is a probability distribution known as a kernel that is parameterized by the sample $x_i$ and by some matrix factor $\*C_i$ that represents a measure of local covariance that is defined later.
We assume the kernel $\mathcal{K}$ satisfies the following properties,
\begin{equation}\label{eq:kernel-properties}
    \begin{gathered}
        \mathcal{K}(x) \geq 0\quad \forall x \in \mathbb{R}^n\\
        \int_{\mathbb{R}^n} \mathcal{K}(x) \mathrm{d}x = 1,\\
        \int_{\mathbb{R}^n} x \mathcal{K}(x) \mathrm{d}x = 0,\\
        \int_{\mathbb{R}^n} x x^{\mathsf{T}} \mathcal{K}(x) \mathrm{d}x < \infty,\\
        \mathcal{K}(x; \mu, \*H) \coloneqq \frac{1}{\sqrt{\lvert\*H\rvert}}\mathcal{K}\left(\*H^{-1/2}( x - \mu)\right),\quad \,\*H > 0\\
        \norm{x} = \norm{y} \implies \mathcal{K}(x) = \mathcal{K}(y), \quad \forall x,y\in\mathbb{R}^n,
    \end{gathered}
\end{equation}
where the first and second properties properties ensure that the Kernel is a probability density, the third property ensures that the kernel is zero mean, the fourth property ensures that the kernel has a finite (co-)variance, the fifth property defines arbitrary shifting by mean and re-scaling by symmetric positive definite matrices, and the sixth property ensures that the kernel is radially symmetric.

Assume that we are given a nonlinear measurement of the state,
\begin{equation}\label{eq:measurement}
    y = h(x) + \eta,
\end{equation}
through some non-linear function $h$, and
with unbiased, $\mathbb{E}[\eta] = 0$, additive error that has covariance $\operatorname{Cov}[\eta] = \*R$.
For the remainder of this work we make the common, but not required assumption that the measurement likelihood is Gaussian,
\begin{equation}\label{eq:measurement-likelihood}
    p(y | x) = \mathcal{N}(y\,;\, h(x), \*R),
\end{equation}
though this assumption can be relaxed and generalized in various different ways that are outside the scope of this work.

We want to leverage this general framework to perform Bayesian inference, thus our goal is to find a representation of the posterior distribution,
\begin{equation}\label{eq:Bayes-rule}
    p(x | y) \propto p(y | x) p(x),
\end{equation}
which in the case of the prior mixture model~\cref{eq:KDE-estimate} and measurement likelihood~\cref{eq:measurement-likelihood} is exactly described by the mixture model,
\begin{equation}\label{eq:posterior-MM}
    p(x | y) = \sum_{i = 1}^{N} w_{i}\, \widetilde{\mathcal{K}}(x\,;\, x_i, \*C_i, y, \*R),
\end{equation}
where the new mixture components are given by
\begin{equation}\label{eq:posterior-mixture-mode}
    \begin{multlined}
    \widetilde{\mathcal{K}}(x\,;\, x_i, \*C_i, y, \*R) =\qquad\qquad\qquad\qquad\qquad\qquad\\ 
    \frac{\mathcal{K}(x\,;\, x_i, \*C_i)\mathcal{N}(y\,;\, h(x), \*R)}{\int_{\mathbb{R}^n} \mathcal{K}(x\,;\, x_i, \*C_i)\mathcal{N}(y\,;\, h(x), \*R)\,\mathrm{d} x},
    \end{multlined}
\end{equation}
with the weights,
\begin{equation}\label{eq:posterior-weights}
    w_{i} \propto \int_{\mathbb{R}^n} \mathcal{K}(x\,;\, x_i, \*C_i)\mathcal{N}(y\,;\, h(x), \*R) \,\mathrm{d} x,
\end{equation}
which are derived by simple convolution of distributions.

If our probability of interest $f$ has known finite covariance,
\begin{equation}\label{eq:true-covariance}
    \operatorname{Cov}(f) = \*\Sigma_f < \infty,
\end{equation}
then the standard scalar parameterization~\cite{silverman2018density} of the covariances, $\*C_i$, in~\cref{eq:KDE-estimate}, is given by,
\begin{equation}\label{eq:scaled-covariance}
    \*C_i = h^2 \*\Sigma_{f}, \quad i = 1,\dots,N,
\end{equation}
where $h$ is known as the bandwidth parameter.
In this work we make use of the scalar parameterization to restrict the parameters of our kernel density estimate to two choices: the kernel $\mathcal{K}$ and the bandwidth $h$. 

The most common type of kernel that is used is the (zero-mean, unit-covariance-scaled) Gaussian kernel,
\begin{equation}\label{eq:Gaussian-kernel}
    \mathcal{N}(x) = \frac{1}{(2\pi)^{n/2}} e^{-\frac{1}{2} x^{\mathsf{T}} x},
\end{equation}
which has many nice properties that make it simple to use and reason about. These very same properties make it extremely attractive to the practitioner, and make its choice a thought-terminating clich\'e---the mere invocation of the Gaussian kernel assumption terminates reasoning without significant pushback. We now show why the Gaussian kernel assumption needs to be questioned---especially in the high-dimensional setting.

\subsection{Minimizing KDE Error}

Let's now put our focus on the goodness-of-fit of the kernel density estimate~\cref{eq:KDE-estimate}.
The most common metric that describes how well the KDE estimate approximates the target distribution is the mean integral squared error (MISE), 
\begin{equation}\label{eq:MISE}
    \operatorname{MISE}\left(f, \widetilde{f}_{\*X_N}\right) = \mathbb{E}_{\*X_N}\left[\int_{\mathbb{R}^n} \left(f(x) - \widetilde{f}_{\*X_N}(x)\right)^2\mathbf{d}x\right]
\end{equation}
which describes the squared error over all the support of the target distribution $f$ averaged over all possible realizations of $N$ samples $\*X_N$. Dealing with the MISE directly is intractable for most kernels and distributions of interest, thus the asymptotic MISE (AMISE),
\begin{equation}\label{eq:AMISE}
    \operatorname{AMISE}\left(f, \widetilde{f}_{\*X_N}\right) = \frac{1}{4} h^4 \alpha^2 \gamma + N^{-1} h^{-n} \beta, 
\end{equation}
is frequently used instead.
The new parameters of~\cref{eq:AMISE} are given by,
\begin{equation}\label{eq:AMISE-parameters}
\begin{aligned}
    \alpha &= \frac{1}{n}\operatorname{tr}\left(\int_{\mathbb{R}^n} x^{\mathsf{T}} x \,\mathcal{K}(x)\mathrm{d} x\right)\\
    \beta &= \int_{\mathbb{R}^n} \mathcal{K}(x)^2 \mathrm{d} x,\\
    \gamma &= \int_{\mathbb{R}^n} \tr^2\left[\nabla_x^2 \widehat{f}(x)\right] \mathrm{d} x,
\end{aligned}
\end{equation}
where $\widehat{f}$ is the scaling of $f$ by $\*\Sigma_f$ that is ameanable to representation by the unscaled kernel $K$. The derivation of the above can be found in~\cite{silverman2018density}.

The following lemma (a wonderfully simple proof of which is given in~\cite{baker1999integration}) is useful for proving results about the parameters in~\cref{eq:AMISE-parameters}.
\begin{lemma}[Radially symmetric integral]\label{lem:radially-symmetric integral}
For a radially symmetric function, $f:\mathbb{R}\to\mathbb{R}$, the integral of $f(x^{\mathsf{T}}x)$ over the ball $B_n(R)$ in $n$ dimensions with radius $R$, has a alternative representation in terms of the scalar variable, $r$,
\begin{equation}
\int_{B_n(R)}\mkern-36mu f(x^{\mathsf{T}} x) \mathrm{d} x = \int_0^{R}\mkern-12mu n c_n f(r^2) r^{n-1} \mathrm{d}r,
\end{equation}
where,
\begin{equation}
c_n = \frac{\pi^{\frac{n}{2}}}{\Gamma\left(\frac{n}{2} + 1\right)},
\end{equation}
is the volume of a unit sphere in $n$ dimensions.
\end{lemma}

By lemma~\ref{lem:radially-symmetric integral}, if the target distribution $\widehat{f}$ in~\cref{eq:AMISE-parameters} is the unit-covariance Gaussian, then the last parameter simplifies to,
\begin{equation}\label{eq:Gaussian-gamma}
    \gamma = \frac{1}{2^n \sqrt{\pi}^n}\left(\frac{1}{2} n + \frac{1}{4}n^2\right),
\end{equation}
which---while not required to be defined for any of the subsequent derivations---is the value that is used for the rest of this work, as exploring alternatives is outside the current scope.

Given an arbitrary kernel $\mathcal{K}$ we want create a kernel density estimate that minimizes the AMISE~\cref{eq:AMISE}. In effect, this means that we want choose the bandwidth parameter in~\cref{eq:scaled-covariance} that is optimal in terms of the error.
The following result about the optimal bandwidth is from~\cite{silverman2018density}:
\begin{theorem}
    The optimal bandwidth in~\cref{eq:scaled-covariance} that minimizes the AMISE in~\cref{eq:AMISE-parameters}, is given by,
    \begin{equation}\label{eq:optimal-bandwidth}
        h = \left[\frac{\beta n}{\alpha^2 \gamma N}\right]^{\frac{1}{n + 4}},
    \end{equation}
    where $n$ is the dimension of the system, $N$ is the number of samples and the rest of the parameters are defined by~\cref{eq:AMISE-parameters}.
\end{theorem}

We now go the other direction. Fixing the optimal bandwidth to be~\cref{eq:optimal-bandwidth}, we want to find the kernel that minimizes the AMISE.
Substituting the optimal bandwidth~\cref{eq:optimal-bandwidth} into the AMISE error metric~\cref{eq:AMISE} we get, 
\begin{equation}\label{eq:AMISE-with-optimal-bandwidth}
    \operatorname{AMISE}\left(f, \widetilde{f}_{\*X_N}\right) = \underbrace{\frac{(n + 4)}{4}\gamma^{\frac{n}{n + 4}}}_{\text{reference dist.}}\underbrace{\beta\left(\frac{n\beta}{\alpha^2}\right)^{-\frac{n}{n+4}}}_{C(\mathcal{K})}
    \underbrace{N^{-\frac{4}{n+4}}}_{\text{conv. rate}},
\end{equation}
where the first term is purely a function of the target distribution (if known), or is a function of the mismatch between the true distribution and a reference used to compute the term $\gamma$. The third term in~\cref{eq:AMISE-with-optimal-bandwidth} is the rate of converengece in the number of samples $N$, which becomes slower and severely sub-linear as the dimension $n$ increases. The term that we are interested in is the second term, $C(\mathcal{K})$, which is purely dependent on the choice of kernel $\mathcal{K}$. 
As the rate of convergence is sub-linear, the scaling term $C(\mathcal{K})$ plays a significant role in the convergence of the kernel density estimation method.
This means that the choice of kernel is the sole choice that fully determines the error of the KDE.

\begin{figure}[t]
    \centering
    \subfigure[Gaussian distribution]{
    \includegraphics[width=0.99\linewidth]{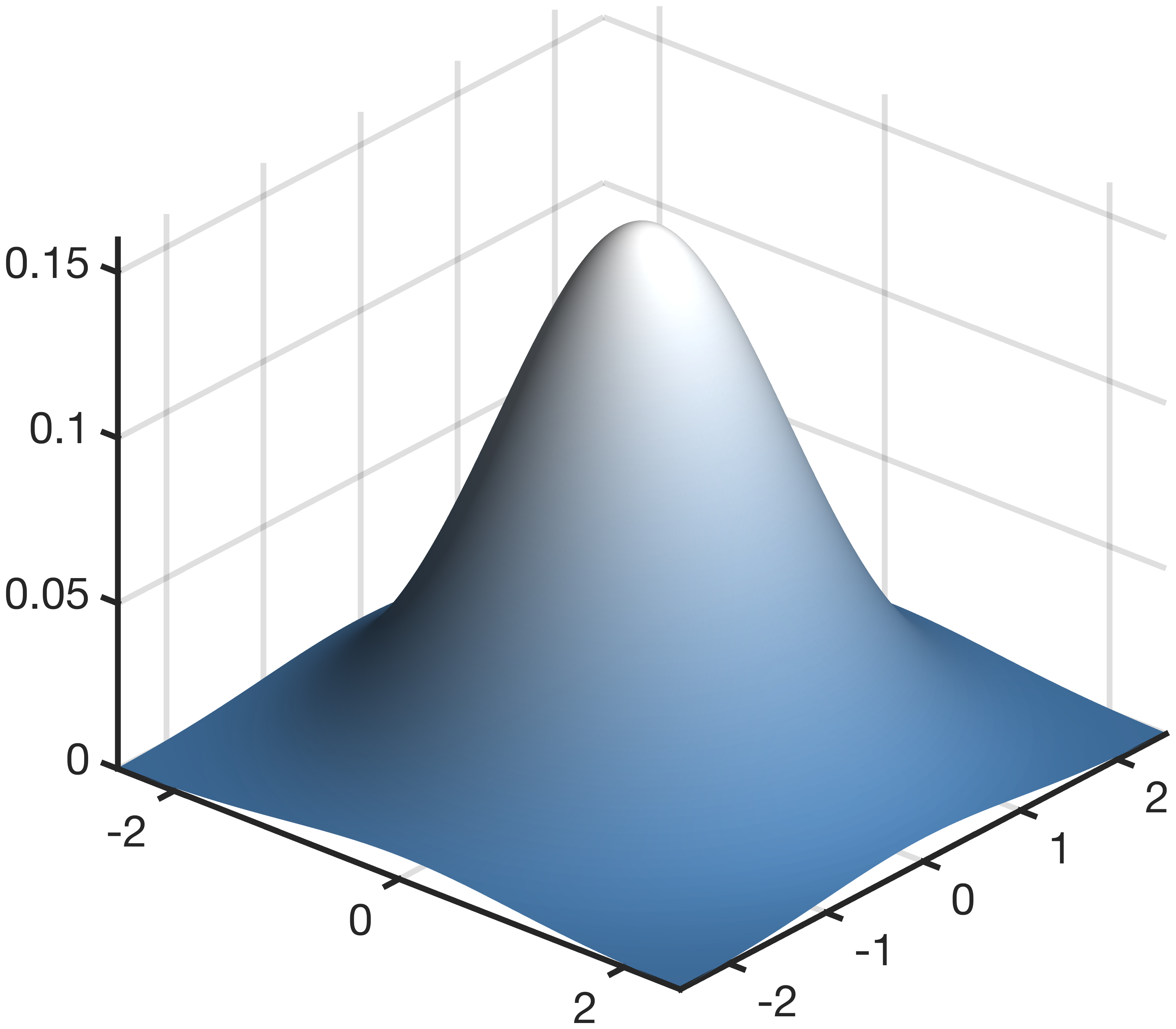}}\\
    \subfigure[Epanechnikov distribution]{\includegraphics[width=0.99\linewidth]{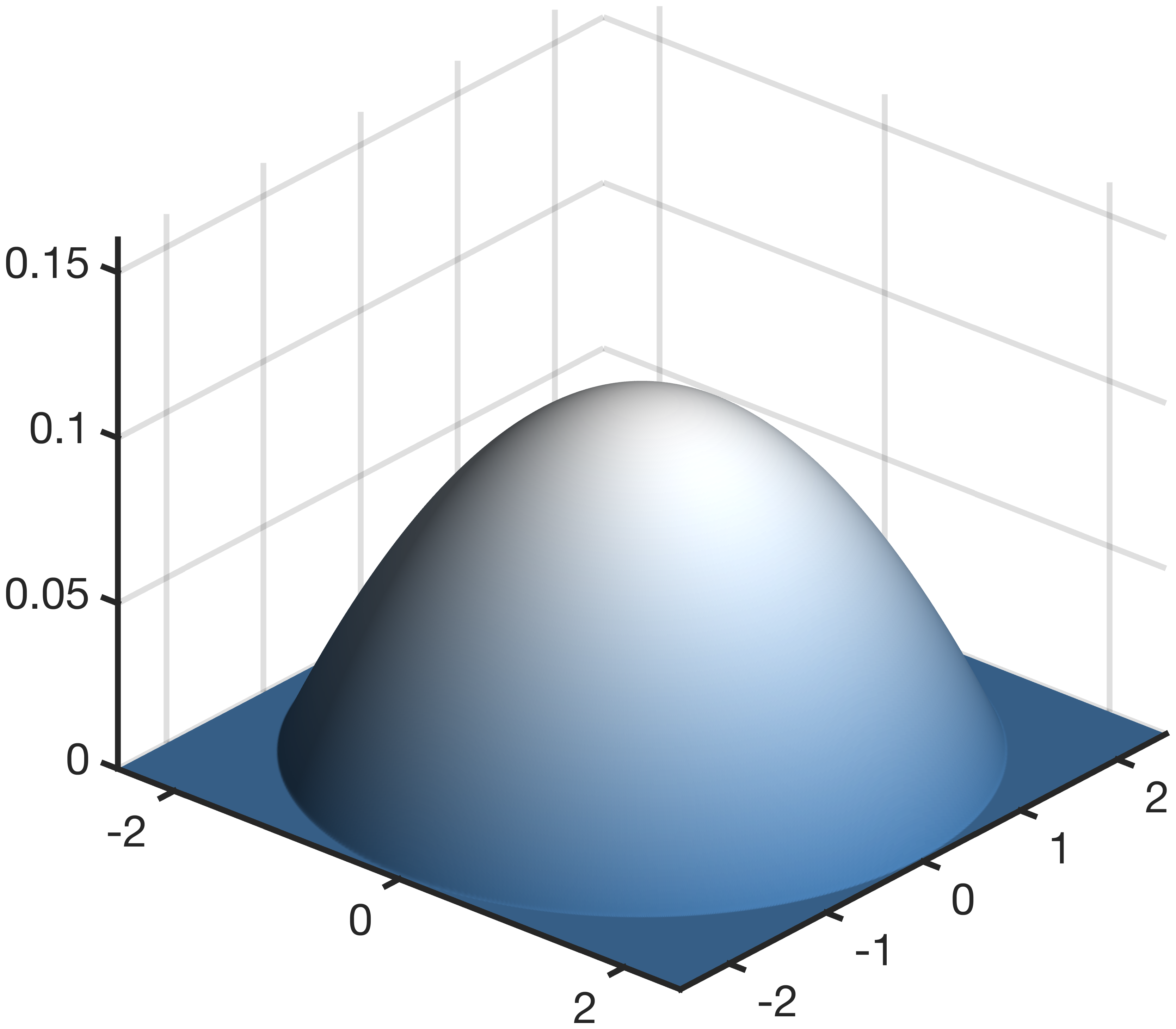}}
    \caption{(a) Gaussian and (b) Epanechnikov distribution for $n=2$ spatial dimensions, mean of zero and identity covariance.}
    \label{fig:Epanechnikov-distribution}
\end{figure}

The kernel that minimizes $C(\mathcal{K})$, is given by~\cite{deheuvels1977estimation,epanechnikov1969non} and is the unit-covariance-scaled Epanechnikov distribution,
\begin{equation}\label{eq:Epanechnikov-distribution}
    \mathcal{E}(x) = \frac{n + 2}{2 c_n (n+4)^{\frac{n+2}{2}}}(n + 4 - x^{\mathsf{T}} x), \quad x^{\mathsf{T}} x < n+4,
\end{equation}
which has support on the ball $B_n(\sqrt{n+4})$.
A visualization of the Epanechnikov distribution~\cref{eq:Epanechnikov-distribution} compared to the Gaussian distribution~\cref{eq:Gaussian-kernel} is presented in~\cref{fig:Epanechnikov-distribution}.

As~\cref{eq:Epanechnikov-distribution} is the optimal kernel that minimizes the AMISE~\cref{eq:AMISE-with-optimal-bandwidth} given the optimal bandwidth, it is reasonable to ask, just how much accuracy are we sacrificing when we choose a suboptimal kernel?
We can quantify the how effective a given kernel is relative to the Epanechnikov kernel by calculating the 
the efficiency~\cite{silverman2018density} of an arbitrary kernel $\mathcal{K}$,
\begin{equation}\label{eq:efficiency}
    \operatorname{eff}(\mathcal{K}) = \left(\frac{C(\mathcal{E})}{C(\mathcal{K})}\right)^{\frac{n+4}{4}},
\end{equation}
which describes a scaling of the effective ensemble size of the kernel $\mathcal{K}$ with respect to the optimal Epanechnikov distribution $\mathcal{E}$. The power $(n+4)/4$ represents the inverse of the rate of convergence in~\cref{eq:AMISE-with-optimal-bandwidth}.
In other words, this means  that the error using $N$ samples and the kernel $\mathcal{K}$ is equivalent to using $N\operatorname{eff}(\mathcal{K})$ samples and the kernel $\mathcal{E}$.
Conversely this also means that the error of using the Epanechnikov kernel $\mathcal{E}$ with $N$ samples is equivalent to using the kernel $\mathcal{K}$ with $N/\operatorname{eff}(\mathcal{K})$ samples.
Note that the efficiency~\cref{eq:efficiency} is highly dependent on the dimension $n$, thus in the worst-case the efficiency could effectively be zero for a large enough $n$.

Before we give a closed form expression for the efficiency of the Gaussian kernel, we have to derive a few constants for the Gaussian~\cref{eq:Gaussian-kernel} and Epanechnikov~\cref{eq:Epanechnikov-distribution} kernels.

\begin{lemma}\label{lem:Gaussian-kernel-constants}
    By lemma~\ref{lem:radially-symmetric integral}, for the Gaussian kernel,
    \begin{equation}
        \alpha_{\mathcal{N}} = 1, \quad \beta_{\mathcal{N}} = \frac{1}{(2\sqrt{\pi})^n},
    \end{equation}
    are the constants defined by~\cref{eq:AMISE-parameters}.
\end{lemma}

\begin{lemma}\label{lem:Epanechnikov-kernel-constants}
    Again, by lemma~\ref{lem:radially-symmetric integral}, for the Epanechnikov kernel,
    \begin{equation}
        \alpha_{\mathcal{E}} = 1,\quad \beta_{\mathcal{E}} = \frac{2}{c_n} (n+2) (n+4)^{-\frac{n}{2}-1},
    \end{equation}
    are the constants defined by~\cref{eq:AMISE-parameters}.
\end{lemma}

We now have the tools necessary to derive the efficiency of the Gaussian kernel~\cref{eq:Gaussian-kernel} for $n$-dimensional KDE.

\begin{theorem}\label{thm:Gaussian-kernel-efficiency}
    The efficiency~\cref{eq:efficiency} of the Gaussian kernel is,
    \begin{equation}
        \operatorname{eff}(\mathcal{N}) = \frac{2^{n+2}}{(n+4)^{\frac{n}{2}+1}}\Gamma \left(\frac{n}{2}+2\right),
    \end{equation}
\end{theorem}
\begin{proof}
    This is a direct application of the constants in lemma~\ref{lem:Gaussian-kernel-constants} and lemma~\ref{lem:Epanechnikov-kernel-constants} to~\cref{eq:efficiency}.
\end{proof}

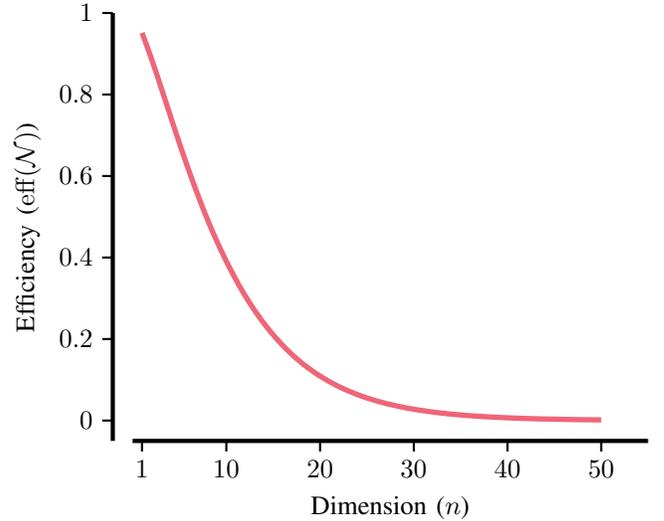
\begin{figure}
\centering
\begin{tikzpicture}
\begin{axis}[clean,
        xtick={1, 10, 20, 30, 40, 50},
        table/col sep=comma,
        xmin = 0,
        xmax = 55,
        ymin = -0.05,
        ymax = 1,
        clip mode=individual,
        xlabel = {Dimension ($n$)},
        ylabel = {Efficiency ($\operatorname{eff}(\mathcal{N})$)},
        every axis plot/.append style={line width=2pt, mark size=3.5pt}]
\addplot[color=tolred,domain=1:50, samples=200]{(2^(x+2))*((x+4)^(-x/2-1))*sqrt(pi*x + 2*pi)*(((x+2)/(2*exp(1)))^(x/2 + 1))*(((x/2+1)*sinh(2/(x+2)))^((x + 2)/4))};
\end{axis}
\end{tikzpicture}
\caption{Efficiency of the Gaussian kernel relative to the dimension~$n$. Values of the efficiency for non-integer dimensions are plotted for completeness.}
\label{fig:Gaussian-kernel-efficiency}
\end{figure}

The efficiency of the Gaussian kernel is plotted in~\cref{fig:Gaussian-kernel-efficiency}. Notice that the rate of decay of the efficiency is exponential, and that for $n=40$ the efficiency is well below $1\%$.
This means that the Gaussian kernel is highly inefficient in higher dimensions, and that attempting to use the Epanechnikov kernel might be cost-effective. Moreover, note the following:
\begin{corollary}\label{cor:worst-case}
By simple application of elementary calculus rules, the efficiency of the Gaussian kernel in~\cref{thm:Gaussian-kernel-efficiency} tends towards zero as the dimension of the system $n$ tends towards infinity, 
\begin{equation}
    \lim_{n\to\infty} \operatorname{eff}(\mathcal{N}) = 0.
\end{equation}
\end{corollary}
The result in~\cref{cor:worst-case} means that there is no lower bound on the inefficiency of the Gaussian kernel. For large enough systems, the Gaussian kernel density estimate is essentially meaningless compared to the Epanechnikov kernel estimate.

\subsection{Sampling from a Epanechnikov distribution}
\label{sec:Epanechnikov-sampling}

Being able to sample from the Epanechnikov distribution is important step for sampling from the approximated posterior.
Similar to~\cref{eq:kernel-properties}, we can write the Epanechnikov distribution with mean $\*\mu$ and covariance $\*\Sigma$ as
\begin{equation}\label{eq:generalized-Epanechnikov}
    \mathcal{E}(x ; \mu, \*\Sigma) = \frac{1}{\sqrt{\lvert\*\Sigma\rvert}}\mathcal{E}\left(\*\Sigma^{-\frac{1}{2}}(x - \mu)\right).
\end{equation}
It is known from~\cite{L1book}  that the random variable,
\begin{equation}\label{eq:sampling-EP}
\begin{gathered}
    \varepsilon = \*\mu + \*\Sigma^{\frac{1}{2}}\sqrt{(n + 4) \eta}\,\*T,\\
    \*T \sim \mathcal{U}(\mathcal{S}^{n-1}),\quad \eta \sim \beta\left(\frac{n}{2}, 2\right),
\end{gathered}
\end{equation}
is a sample from the Epanechnikov distribution~\cref{eq:generalized-Epanechnikov} with mean $\*\mu$ and covariance $\*\Sigma$.
We slightly modify \cref{eq:sampling-EP} in order to accomodate further modification later on.
A sample from the Epanechnikov distribution can be created using the following procedure:
\begin{enumerate}
    \item Sample a random $s$ from the unit Gaussian distribution $\mathcal{N}(\*0_n, \*I_{n\times n})$,
    \item project $s$ onto the shell with radius $\sqrt{n + 4}$, to get $\hat{s} = \frac{\sqrt{n+4}}{\norm{s}}s$,
    \item sample the beta-distributed, $\kappa \sim \beta(n/2, 2)$,
    \item combine to get a sample $\varepsilon = \mu + \*\Sigma^{\frac{1}{2}}\kappa\hat{s}$.
\end{enumerate}

\subsection{Illustrative filtering example (banana problem)}
\label{sec:banana}

As an illustrative example throughout the main text, we make use of the $n$-dimensional `banana' problem, which is a generalization of the two-dimensional problem used in~\cite{michaelson2023ensemble}.

Take the prior to have a mean $\mu$ such that,
\begin{equation}
    \mu_i = \begin{cases}
        \alpha & i = \ceil*{i/2}\\
        0 & \text{sonst}
    \end{cases},
\end{equation}
where in this work we take the parameter $\alpha = -2.5$ as opposed to $\alpha = -3.5$ in some previous works.
Take the prior to have covariance $\Sigma$, such that,
\begin{equation}
    \Sigma = \begin{bmatrix}
        1 & 0.5 & 0 & \cdots & 0\\
        0.5 & 1 & 0.5 &   & \vdots\\
        0 & 0.5  & \ddots & \ddots  & 0\\
        \vdots &   & \ddots &  1 & 0.5\\
        0 & \cdots  & 0  & 0.5  & 1
    \end{bmatrix},
\end{equation}
which is one on the main diagonal,  $0.5$ on the two main off-diagonals, and zero elsewhere.
The prior distribution can either be Gaussian or Epanechnikov with the same mean $\mu$ and covariance $\Sigma$ parameters.

Take the measurement to be a magnitude,
\begin{equation}
    h(x) = \norm{x},
\end{equation}
where the standard Euclidean 2-norm is used. The  measurement realization is $y = 1$ about which we have Gaussian uncertainty with variance $R = 0.01$. This is called the banana problem as in two dimensions the posterior resembles a banana fruit, and is a common example of non-Gaussian uncertainty arising in aerospace applications.

\section{Ensemble Gaussian Mixture Filter}

We now present the Ensemble Gaussian mixture filter, which we subsequently generalize to make use of the optimal Epanechnikov mixture.

The filter is presented as three distinct steps: 1) the Kernel density estimation, 2) the Gaussian sum update, and 3) the resampling procedure.

\subsection{Kernel density estimation}
Given an ensemble of $N$ samples from some unknown prior distribution,
\begin{equation}
    \*X^- = \left[x_1^-, x_2^-, \dots, x_N^-\right],
\end{equation}
first approximate the covariance of the distribution,
\begin{equation}\label{eq:covariance-approximation}
    \widetilde{\*\Sigma}^- = \frac{1}{N-1}\*X^-\left(\*I_{N\times N} - \frac{1}{N}\*1_N\*1_N^T\right)\*X^{-,T},
\end{equation}
as a proxy to the known covariance in~\cref{eq:true-covariance}.

Second, build KDE estimate of the distribution as in~\cref{eq:KDE-estimate},
\begin{equation}\label{eq:EnGMF-prior-GMM}
    \widetilde{f}_{\*X^-}(x) = \frac{1}{N} \sum_{i=1}^N \mathcal{N}(x; x_i^-, h_{\mathcal{N}}^2 \widetilde{\*\Sigma}^-),
\end{equation}
by making use of the covariance approximation in~\cref{eq:covariance-approximation} and the optimal bandwidth $h_{\mathcal{N}}$ defined by~\cref{eq:optimal-bandwidth} with constants defined for the normal distribution in lemma~\ref{lem:Gaussian-kernel-constants}.

\subsection{Gaussian sum update}

\begin{figure*}
    \centering
    \includegraphics[width=0.9\linewidth]{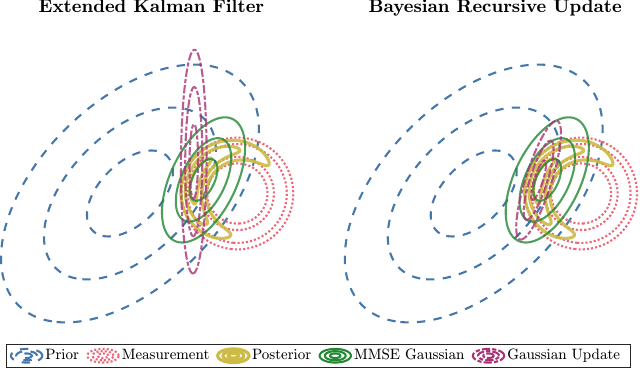}
    \caption{Comparison of Gaussian sum updates for inference on a Gaussian distribution with a non-linear measurement. Three standard deviations of the Gaussian prior are represented by the large dashed blue lines on the the left of the figures.
    The measurement distribution is represented by the dotted red circles
    The minimal mean squared error Gaussian approximation to the posterior is represented by solid green lines.
    The left figure represents an update using the extended Kalman filter, while the right figure represents an update using the Bayesian recursive update.}
    \label{fig:Gaussian-sum-update}
\end{figure*}
 
We can take advantage of the Gaussian mixture structure, to perform an
Gaussian sum update on each one of the components to get an approximation of the posterior. The most straightforward approach to performing a Gaussian update is to make use of the extended Kalman filter from the prior Gaussian mixture~\cref{eq:EnGMF-prior-GMM} to an approximation of the posterior,
\begin{equation}\label{eq:EnGMF-update}
    \begin{aligned}
        x^\sim_{i} &= x^-_{i} - \*G_{i}\left(h(x^-_{i}) - y\right),\\
        h_{\mathcal{N}}^2\widetilde{\*\Sigma}^\sim_{i} &= \left(\*I -  \*G_{i}\*H_{i}^T\right)h_{\mathcal{N}}^2\widetilde{\*\Sigma}^-_{i},\\
        \*G_{i} &= h_{\mathcal{N}}^2\widetilde{\Sigma}^-_{i}\*H_{i}^T{\left(\*H_{i}h_{\mathcal{N}}^2\widetilde{\Sigma}^-_{i}\*H_{i}^T + \*R_i\right)}^{-1},\\
        w_{i} &\propto  \mathcal{N}\left(y\,;\, h(x^-_{i}),\, \*H_{i}h_{\mathcal{N}}^2\widetilde{\*\Sigma}^-_{i}\*H_{i}^T + \*R\right),\\
        \*H_{i} &= \left.\frac{d h}{d x}\right\rvert_{x = x^-_{i}},
    \end{aligned}
\end{equation}
with $\widetilde{x}_i$ representing the mean of the $i$th posterior component, $h_{\mathcal{N}}^2\widetilde{\*\Sigma}^\sim_{i}$ representing its covariance, and $w_i$ its weight.

Thus, the Gaussian mixture,
\begin{equation}\label{eq:EnGMF-posterior-approximation}
    \widetilde{f}_{\*X^+}(x) = \sum_{i=1}^N w_i\,\mathcal{N}(x; x_i^{\sim}, h_{\mathcal{N}}^2\widetilde{\*\Sigma}^\sim_i),
\end{equation}
is an approximation to the posterior. Note that as the GMM update~\cref{eq:EnGMF-update} is exact when the measurement $h$ in~\cref{eq:measurement} is linear, the measurement error $\eta$ is Gaussian, and when the prior mixture model is a Gaussian mixture. 

When the measurements are non-linear, the update in~\cref{eq:EnGMF-update} is equivalent to performing a extended Kalman filter (EKF) on each one of the components, which is not equal to the optimal Gaussian that describes the (local) posterior distribution.

In order to alleviate some of the error from this approximation, in this work we make use of the Bayesian recursive update filter (BRUF)~\cite{michaelson2023ensemble, michaelson2023recursive}, which was previously used for the EnGMF in~\cite{durant2024you}.

In the BRUF, the EKF-based update in~\cref{eq:EnGMF-update} is essentially $M$  iterations of the EKF with an inflated measurement covariance $\*R_i \xleftarrow[]{}M\*R_i$. For more information on the BRUF as it applies to the EnGMF see~\cite{durant2024you}.

A visual representation of the difference between the EKF based update and the BRUF update, with $M=5$. on the two-dimensional banana example~\cref{sec:banana} is provided in in~\cref{fig:Gaussian-sum-update}. The optimal, in KL-divergence, Gaussian is also plotted as a reference. As can be seen, the BRUF is capable of both representing the posterior mean and posterior covariane in a much more correct fashion, as it re-linearizes the Jacobian of the measurement around each successive best-estimate.

\subsection{Resampling}

The final step of the EnGMF is to resample from the posterior approximation~\cref{eq:EnGMF-posterior-approximation} through the use of standard techniques:
\begin{enumerate}
    \item first sample from the probability mass function defined by the weights $\{w_i\}_{i=1}^N$, to get a mode $j$,
    \item then sample from the normal distribution  defined by the mode, $\mathcal{N}(x_j^\sim, h_{\mathcal{N}}^2\widetilde{\*\Sigma}^\sim_i)$,
    \item repeat as many times as samples are required.
\end{enumerate}
Of note is that the resampling procedure above does not produce independent and identically distributed samples, merely exchangeable samples, as the means, covariances, and weights are themselves random variables on which the samples are conditioned. In the authors' experience this is more of a technical problem and not one which contributes any significant error in practice.

We make use of all of the machinery of the EnGMF to generalize to a Epanechnikov kernel mixture.

\section{Ensemble Epanechnikov Mixture Filter}
\label{sec:EnEMF}

\begin{figure*}
    \centering
    \includegraphics[width=0.9\linewidth]{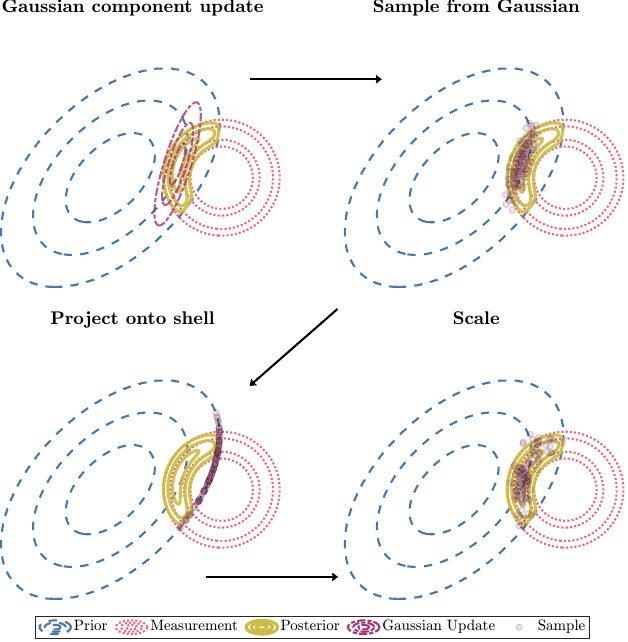}
    \caption{A visual description of the approximated posterior Epanechnikov sampling procedure for one Epanechnikov mode. First, in the top left panel the Gaussian component update from the prior to the candidate posterior is performed using the measurement likelihood with samples taken therefrom in the top right panel. Next, in the bottom left panel the samples are projected onto the shell of the prior Epanechnikov component, and finally the projected samples are randomly scaled in the bottom left panel.}
    \label{fig:sampling-procedure}
\end{figure*}

We now present the ensemble Epanechnikov mixture filter as a generalization of the ensemble Gaussian mixture filter to the more efficient Epanechnikov mixture model.

Just like the EnGMF, the first step of the EnEMF is to find the statistical covariance~\cref{eq:covariance-approximation}. The second step is to perform kernel density estimation,
\begin{equation}\label{eq:prior-EMM}
    \widetilde{f}_{\*X^-}(x) = \sum_{i=1}^N \frac{1}{N}\mathcal{E}(x; x_i^-, h^2_\mathcal{E}\widetilde{\*\Sigma}^-_i)
\end{equation}
but this time using the Epanechnikov distribution~\cref{eq:generalized-Epanechnikov} and the Epanechnikov bandwidth from lemma~\ref{lem:Epanechnikov-kernel-constants}.

Given the~\cref{eq:prior-EMM} the exact posterior distribution is given by
\begin{equation}\label{eq:EMM-posterior}
    \widetilde{f}_{\*X^+}(x) = \sum_{i=1}^N w_i\,\widetilde{\mathcal{E}}\left(x; x_i^-, h^2_\mathcal{E}\widetilde{\*\Sigma}^-_i, y, \*R\right),
\end{equation}
where the distribution above is of the form~\cref{eq:posterior-mixture-mode},
\begin{equation}\label{eq:EMM-posterior-mode}
    \widetilde{\mathcal{E}}\left(x; x_i^-, h^2_\mathcal{E}\widetilde{\*\Sigma}^-_i, y, \*R\right) \propto \begin{multlined}
        \mathcal{E}(x; x_i^-, h^2_\mathcal{E}\widetilde{\*\Sigma}^-_i)\\
        \quad\cdot\mathcal{N}(y\,;\, h(x_i^-), \*R) 
    \end{multlined},
\end{equation}
which does not have a simple closed form representation.
We show later that it is possible to sample from~\cref{eq:EMM-posterior-mode} by making use of the Gaussian sum update~\cref{eq:EnGMF-posterior-approximation} either in the EKF or the BRUF variant,
where the terms are almost identical to the EnGMF update~\cref{eq:EnGMF-update} except for the bandwidth factor $h_{\mathcal{E}}$. The weights in the Gaussian sum update now require special attention, in the next section.

\subsection{Weights}

For the weights $\{w_i\}_{i=1}^N$ in~\cref{eq:EMM-posterior}, we can (i) again use the same weights defined in the Gaussian sum update~\cref{eq:EnGMF-update}, (ii) attempt to exactly compute the integral in~\cref{eq:posterior-weights}, or (iii) approximate the integral through a closed form Gaussian solution.
In the authors' experience, (i) does not produce accurate results, and (ii) incurs a significant computational cost in the high-dimensional setting, therefore (iii) is the most reasonable option for now.

\subsubsection{Gaussian weight approximation}
The first type of approximation that we look at is through approximating the weights by a distribution that is Gaussian.

\begin{theorem}\label{thm:EK-approximation}
    A first order approximation to Epanechnikov kernel is given by a Gaussian approximation with covariance $\frac{n+4}{2}\*\Sigma$.
\end{theorem}
\begin{proof}
Observe the series expansion of the logarithm
\begin{equation}
\begin{multlined}
    \log\left(n + 4 - (x - \mu)^T \*\Sigma^{-1} (x - \mu)\right) \\ \,\,=\log\left(n+4\right) - \frac{1}{n+4}(x - \mu)^T \*\Sigma^{-1} (x - \mu) \\
    + \text{h.o.t.},
\end{multlined}
\end{equation}
meaning that the Epanechnikov distribution can be approximated by the Gaussian distribution,
\begin{equation}
    \mathcal{E}\left(x; \mu, \*\Sigma\right) \approx \mathcal{N}\left(x; \mu, \frac{n+4}{2}\Sigma\right),
\end{equation}
as required.
\end{proof}

Therefore by using the derivation in~\cref{thm:EK-approximation} for the weights  it is possible to approximate the weights of the EnEMF update as,
\begin{equation}\label{eq:EnEMF-Gaussian-weight-approximation}
    w_i \propto \mathcal{N}(y\,;\, h(x^-_i), \*H_i \frac{s_{\mathcal{E}} h^2_{\mathcal{E}}(n+4)}{2}\widetilde{\*\Sigma}^-_i\*H_i^T + \*R),
\end{equation}
where $s_{\mathcal{E}}$ is scaling factor to help taper the effect of underweighting. With an optimally tuned scaling factor $s_{\mathcal{E}}$, the weight update approximation in~\cref{eq:EnEMF-Gaussian-weight-approximation} should approach the true mixture model weight update in~\cref{eq:posterior-weights}.
Importantly, the weight update in~\cref{eq:EnEMF-Gaussian-weight-approximation} has the same computational cost as the weight update in the EnGMF~\cref{eq:EnGMF-update}, and leverages a very similar implementation---the only difference involving a slightly modified covariance. As an additional consideration the weight update in~\cref{eq:EnEMF-Gaussian-weight-approximation} is compatible with convergence of the EnEMF in the limit of particle number, though a proof of this is ancillary to this work.

\subsubsection{Unscented weight approximation}

An alternative approximation to the weights is by using the unscented transform~\cite{durant2024you,sarkka2023bayesian}. Take the $\sigma$-points,
\begin{equation}
\begin{aligned}
    \Xi_0 &= \mu,\\
    \Xi_j &= \mu + \sqrt{n + \lambda}\left[\sqrt{\*\Sigma}\right]_i, \quad j = 1, \dots, n,\\
    \Xi_{j+n} &= \mu - \sqrt{n + \lambda}\left[\sqrt{\*\Sigma}\right]_i, \quad j = 1, \dots, n,
\end{aligned}
\end{equation}
where $\lambda$ is,
\begin{equation}
    \lambda = \alpha^2(n + \kappa) - n,
\end{equation}
with parameters $\alpha$ and $\kappa$. In this work we take $\alpha = 1$ and $\kappa = 3$. 

The Gaussian $\sigma$-points can be transformed into Epanechnikov by observing from~\cref{eq:sampling-EP} that in any direction $\*t$ from the mean $\mu$, the marginal of a Gaussian is a scalar half-Gaussian distribution, and the marginal of a Epanechnikov distribution is beta-distributed. Through this conversion we can arrive at the Epanechnikov $\sigma$-points, $\Xi^{\mathcal{E}}$.
Numerically,
\begin{equation}
    \begin{gathered}
        m_j = \norm{\sqrt{\*\Sigma}^{-1}(\Xi_j - \mu)}, \*t_j = \sqrt{\*\Sigma}^{-1}(\Xi_j - \mu)/m_j,\\
        \Xi^{\mathcal{E}}_j = \sqrt{n+4}P_b^{-1}(P_{h}(m_j))\*t_j
    \end{gathered}
\end{equation}
where $P_{h}$ is the scalar half-Gaussian cumulative distribution function with mean zero and unit covariance, and $P_b$ is the beta cumulative distribution function with parameters $\text{Beta}(n/2, 2)$.

The weights are proportional to the expected value of the likelihood, which is approximated through the sigma points,
\begin{equation}\label{eq:EnEMF-unscented-weights-approximation}
\begin{aligned}
    w_i &\propto \mathbb{E}_{x\sim p(x)}[p(y|X=x_i)],\\
    &\approx \sum_{j = 0}^{2n} W_j \mathcal{N}(y ; h(\Xi^{\mathcal{E}}_j), \*\Sigma_y + \*R),
\end{aligned}
\end{equation}
with the covariance
\begin{equation}
\begin{gathered}
    \*\Sigma_y = \sum_{j = 0}^{2 n} W_j^C \left(h(\Xi^{\mathcal{E}}_j) - \overline{h(\Xi^{\mathcal{E}}_j)} \right)\left(h(\Xi^{\mathcal{E}}_j) - \overline{h(\Xi^{\mathcal{E}}_j)} \right)^T,\\
    \overline{h(\Xi^{\mathcal{E}}_j)} = \sum_{j = 0}^{2n} W_j h(\Xi^{\mathcal{E}}_j),
\end{gathered}
\end{equation}
and where the $\sigma$-point weights are defined as,
\begin{equation}
\begin{gathered}
    W_0 = \frac{\lambda}{\lambda + n},\, W_0^C = \frac{\lambda}{\lambda + n} + 1 - \alpha^2 + \beta,
    \\
    W_j = W_j^C = \frac{1}{2(\lambda + n)},\, j = 1, \dots 2n,
\end{gathered}
\end{equation}
with the parameter $\beta$ which for this work is fixed to $\beta = 2$.

\subsection{Resampling}
\label{sec:resampling}

The most important step is the resampling procedure, as it is the one that is most heavily modified from that of the EnGMF, and relies on the Epanechnikov sampling procedure described in~\cref{sec:Epanechnikov-sampling}.
We can perform resampling from the EnEMF posterior~\cref{eq:EMM-posterior} in the following way:
\begin{enumerate}
    \item generate a sample from the discrete distribution defined by the weights $\{w_i\}_{i=1}^N$, which defines the mode $j$,
    \item generate a sample from the $j$th mode of the Gaussian distribution defined by the Gaussian sum update 
    \begin{equation}
        u \sim \mathcal{N}(x_j^{\sim}, h_{\mathcal{E}}^2\widetilde{\*\Sigma}^\sim_j),
    \end{equation}
    \item project $u$ onto the unit shell defined by the prior mode $\mathcal{E}(x_j^-, h_{\mathcal{E}}^2\widetilde{\*\Sigma}^-_j)$,
    \begin{equation}
    \begin{aligned}
        s &= \left(h_{\mathcal{E}}^2\widetilde{\*\Sigma}^-_j\right)^{-\frac{1}{2}}\!\left(u - x_j^-\right),\\
        \hat{s} &=  \frac{\sqrt{n + 4}}{\norm{s}}s,
    \end{aligned}
    \end{equation}
    to find the sample direction $\hat{s}$ relative to the prior mode mean $x_j^-$, meaning that $x_j^- + \left(h_{\mathcal{E}}^2\widetilde{\*\Sigma}^-_j\right)^{\frac{1}{2}}\!\!\hat{s}$ lies on the boundary of the prior mode, then
    \item sample the variable $\kappa$ from the modified beta distribution in the direction $\hat{s}$
    \begin{equation}\begin{multlined}\label{eq:posterior-magnitude}
        \kappa \sim \frac{1}{2}(n+2) z^{n-1} (1 - z^2)\\\qquad\cdot
        \mathcal{N}\left(y\,;\, h\left(x_j^- + \left(h_{\mathcal{E}}^2\widetilde{\*\Sigma}^-_j\right)^{\frac{1}{2}}\! z \hat{s}\right), \*R\right),
    \end{multlined}
    \end{equation}
    in the scalar variable $0 \leq z < 1$, through an inverse CDF method,
    \item and finally, combine with the prior mean and covariance to get a sample
    \begin{equation}
        \varepsilon = x_j^- + \left(h_{\mathcal{E}}^2\widetilde{\*\Sigma}^-_j\right)^{\frac{1}{2}}\kappa\hat{s}.
    \end{equation}
\end{enumerate}
A visual interpretation of the resampling procedure can be found in~\cref{fig:sampling-procedure}. The resampling procedure is visualized using the 2-dimensional banana problem~\cref{sec:banana} with the Gaussian prior replaced by a Epanechnikov prior with the same mean and covariance.

We now prove that the resampling procedure above is exact when the measurement operator $h$ is linear, as mentioned before, a Gaussian sum update is exact only when the measurement is linear. We start by showing that the Gaussian sum update can be used to describe the distribution of the direction of the posterior.
\begin{lemma}[Posterior ellipsoid]
\label{lem:posterior-ellipsoid}
Given prior knowledge $T$ that is uniformly distributed on the ellipsoid,
\begin{equation}\label{eq:ellipsoid}
    S^{n-1}_{\mu,\*\Sigma} = \left\{x \,\middle|\, (x - \mu)^T\*\Sigma^{-1}(x - \mu) = 1 \right\},
\end{equation}
with center $\mu$ and shape defined by the symmetric positive definite matrix $\*\Sigma$, and given a measurement $Y=y$ about which our uncertainty is unbiased Gaussian with covariance $\*R$ from the linear measurement operator $\*H$,  then the distribution of the Bayesian posterior is proportional to,
\begin{equation}
    p(T | Y = y) \propto 
    \begin{cases}
        \mathcal{N}(t ; \mu^+, \*\Sigma^+) & t \in S^{n-1}_{\mu,\*\Sigma}\\
        0 & \text{sonst}
    \end{cases},
\end{equation}
where $\mu^+$ and $\*\Sigma^+$ are,
\begin{equation}\label{eq:KF-eq}
\begin{gathered}
    \mu^+ = \mu - \*\Sigma \*H^T \left(\*H\*\Sigma \*H^T + \*R\right)^{-1}(\*H \mu - y),\\
    \*\Sigma^+ = \*\Sigma - \*\Sigma \*H^T \left(\*H\*\Sigma \*H^T + \*R\right)^{-1}\*H\*\Sigma,
\end{gathered}
\end{equation}
from the Kalman filter equations.
\end{lemma}
\begin{proof}
Let $X$ be Gaussian with mean $\mu$ and covariance $\*\Sigma$, and let $\Phi$ represent the information of the constraint onto the ellipsoid $S^{n-1}_{\mu, \*\Sigma}$. The random variable $T = (X | \Phi)$ is uniformly distributed on $S^{n-1}_{\mu, \*\Sigma}$.
Conditioning by the measurement, we can observe that 
\begin{equation*}
    (T | Y = y) = (X | \Phi, Y = y),
\end{equation*}
and, as the order of conditioning does not matter, the distribution of the random variable $X^+ = (X | Y = y)$, can be computed first through the Kalman filter equations~\cref{eq:KF-eq}, then projected onto the ellipsoid as,
\begin{equation*}
    (T | Y = y) = (X^+ | \Phi),
\end{equation*}
as required.
\end{proof}

\begin{theorem}
    Given a Epanechnikov distributed random variable $X\sim\mathcal{E}(\mu, \Sigma)$, and a measurement $y$ about which we have Gaussian uncertainty with covariance $\*R$, with linear measurement operator $\*H$, then the procedure in~\cref{sec:resampling} samples from the posterior of the Epanechnikov-times-Gaussian distribution.
\end{theorem}
\begin{proof}
    Recall from~\cref{eq:sampling-EP} that the marginal distribution of the Epanechnikov distribution in any  direction $s$ on the Ellipsoid $S^{n-1}_{\mu, (n+4)\*\Sigma}$ from the center is $\beta(n/2, 2)$. We can therefore write each Epanechnikov sample as the product of a magnitude $\kappa$ and a direction $s$ plus the sum of the ellipsoid center $\mu$. For the direction $\kappa$, its posterior is a trivial application of~\cref{lem:posterior-ellipsoid}. For the magnitude $\kappa$, its scalar distribution is defined simply by~\cref{eq:posterior-magnitude}, as required.
\end{proof}

\section{Numerical Experiments}
\label{sec:numerical-experiment}

The aim of the numerical experiments is two-fold: first, we show that the EnGMF error is highly dependent on the dimension $n$ while the EnEMF erorr is robust to change in $n$, and second, we show that on a dynamical example that the EnEMF requires significantly less samples $N$ than the EnGMF to beat the state-of-the-art ensemble Kalman filter.

\subsection{Banana problem}

\begin{figure}[t]
    \centering
    \begin{tikzpicture}
    \begin{axis}[clean,
        cycle list name=tol,
        xtick={1, 10, 20, 30, 40, 50},
        table/col sep=comma,
        xmin = 0,
        xmax = 51,
        ymin = 0,
        ymax = 1.25,
        clip = true,
        xlabel = {Dimension ($n$)},
        ylabel = {Mean Spatial RMSE},
        every axis plot/.append style={line width=2pt, mark size=3.5pt},
        legend style={at={(0.7,0.88)},anchor=center},
        legend cell align={left}]
    \addplot[mark=o,color=tolblue] table [x=nsT, y=merrsEnEMF, col sep=comma] {data/ndbananaresults.csv};
    \addlegendentry{EnEMF-G ($s_{\mathcal{E}}$ = 0.4)};

    \addplot[mark=star, color=tolred, mark options={solid}] table [x=nsT, y=merrsEnEMFUKF, col sep=comma] {data/ndbananaresults.csv};
    \addlegendentry{EnEMF-U ($s_{\mathcal{E}}$ = 0.5)};

    \addplot [color=tolyellow,mark=+, mark options={solid}] table [x=nsT, y=merrsEnGMF, col sep=comma] {data/ndbananaresults.csv};
    \addlegendentry{EnGMF};

    \addplot [color=tolpurple,mark=triangle, mark options={solid}] table [x=nsT, y=merrsEnKF, col sep=comma] {data/ndbananaresults.csv};
    \addlegendentry{EnKF};

    
    
    \end{axis}
    \end{tikzpicture}
    \caption{A comparison of the EnEMF with the Gaussian weight update (EnEMF-G), with the unscented weight update (EnEMF-U), the EnGMF, and the linearized EnKF for the $n$-dimensional banana problem. The $x$ axis represented the dimension $n$ and the $y$ axis represents the RMSE.}
    \label{fig:ndbanana}
\end{figure}
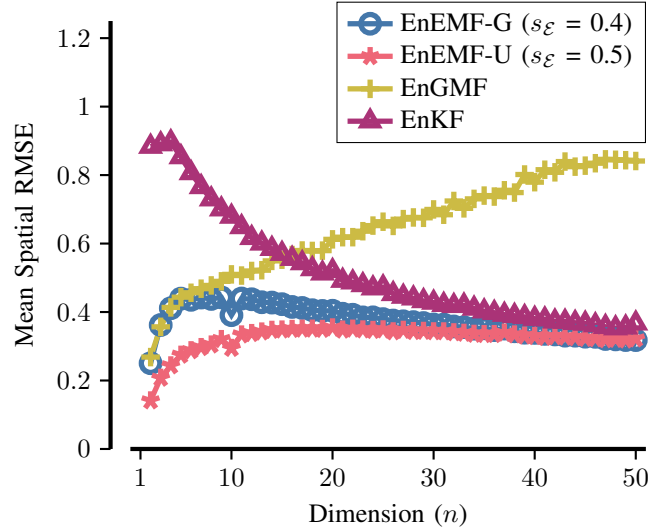

We again make use of the $n$-dimensional banana problem~\cref{sec:banana}. 
For the error metric, we make use of the spatial root-mean-squared error,
\begin{equation}\label{eq:spatial-RMSE}
    \operatorname{RMSE}(\*x^+) = \sqrt{\frac{1}{\lvert\*x\rvert}\sum_{i} \norm{\bar{\*x}^+_{i} - \*x^{\text{true}}_{i}}_2^2},
\end{equation}
where the $\bar{\*x}^+$ is the collection of posterior mean estimates, $\*x^{\text{true}}$ is the collection of true states, and $\lvert\*x\rvert$ is the cardinality of the data.
Starting with fixed number $N = 100$ of iid samples from the prior Gaussian distribution, with limited knowledge of the covariance and mean, the spatial RMSE~\cref{eq:spatial-RMSE} is computed over $500$ different realizations of samples for dimensions $n$ in the range $[1, \dots 50]$. 
We compare the efficacy of the linearized ensemble Kalman filter (EnKF)~\cite{evensen1994sequential,burgers1998analysis}, specifically the linearized Jacobian variant described in~\cite{michaelson2023ensemble}, the EnGMF, the EnEMF with the Gaussian approximation to the weights~\cref{eq:EnEMF-Gaussian-weight-approximation} represented by EnEMF-G, and the EnEMF with the unscented weight approximation~\cref{eq:EnEMF-unscented-weights-approximation} represented by EnEMF-U.
As the EnKF makes use of only the first two statistical moments of the ensemble, and provides an almost linear update, it is a useful baseline for highly non-linear non-Gaussian sequential filtering problems.
The results can be seen in~\cref{fig:ndbanana}.

In general the EnKF error decreases as the dimension of the problem grows larger, meaning that the problem is more and more ameanable to linear filtering algorithms. The EnGMF is worse than the EnKF at about $n=16$, and the error continues to increase with the size of the problem. Both EnEMF algorithms, however, are always strictly better than the EnKF, even for dimensions as high as $n=50$, meaning that the algorithms do not suffer from the curse of dimensionality in the same way as the EnGMF does.

There is a noticeable dip for $n=10$ for the EnEMF algorithms, which is not the result of numerical error, as this dip is consistently present even with significantly more Monte-Carlo samples, thus there is some special set of circumstances that makes $n=10$ unique.

All three mixture model filters use the EKF for the Gaussian sum update instead of the BRUF, as this has been found to be the more numerically stable version of the algorithm for this particular problem.

\subsection{Lorenz '96}

The goal of the second numerical experiment is to show that the EnEMF has the potential to be a superior filter to that of the EnGMF in the high-dimensional setting. We thus make use of the $40$-variable Lorenz '96 equations,
\begin{equation}
    x_k' = -x_{k-1}(x_{k-2} - x_{k+1}) - x_k + F\dots,\quad k = 1,\dots, 40,
\end{equation}
where by the cyclic boundary conditions, $x_0 = x_{40}$, $x_{-1} = x_{39}$, and $x_{41} = x_{1}$. The forcing is set to $F=8$ to have a chaotic system with a Kaplan-Yorke dimension of $27.1$ and $13$ positive Lyapunov exponents~\cite{popov2019bayesian}.

For the non-linear measurement operator we take a magnitude measurement of adjacent variables, 
\begin{equation}
    \left[h(x)\right]_i = \sqrt{x_{2i+1}^2 + x_{2i + 2}^2},\quad i = 1, \dots, 20,
\end{equation}
with an error covariance of $\*R = \frac{1}{4}\*I_{20}$.

For a $40$ variable system the efficiency of the Gaussian kernel is about $0.6\%$, meaning that for a density estimate with $N=100$ samples with the Epanechnikov kernel~\cref{eq:Epanechnikov-distribution}, a sample size of $N=14,\!484$ would be needed with the Gaussian kernel~\cref{eq:Gaussian-kernel}, under an ideal scenario.

In order to account for the fact that the ensemble covariance~\cref{eq:covariance-approximation} is not the true covariance~\cref{eq:true-covariance}, we make use of a standard covariance tapering technique in the data assimilation literature, B-localization~\cite{asch2016data}. For all the filters we take a Gaussian decorrelation function with a localization radius of $r=4$. More details about this technique can be found in~\cite{asch2016data,popov2019bayesian}.
In addition, the EnGMF and both EnEMF algorithms make use of the BRUF with $M=5$ for the Gaussian sum update.

As the EnEMF weight update is not exact, we test the two different choices of the weight calculation with hand-tuned scaling parameters, namely EnEMF-G with $s_{\mathcal{E}} = 0.15$ and EnEMF-U with $s_{\mathcal{E}} = 2.5$.
We additionally test against the EnGMF, and to the EnKF 
to which an inflation~\cite{popov2020explicit} factor of $\alpha_{\text{inf}} = 1.01$ is applied for stability.

We run all algorithms for $192$ independent Monte Carlo simulations (this choice made to be divisible by the 12 physical cores of the machine this experiment was performed on) with $\Delta t = 0.2$ time units between measurements, corresponding roughly to a day in model time. The simulations are run for $2200$ measurements, discarding the first $200$ to account for spinup. The spatio-temporal root mean squared error (RMSE) is used as the error metric for determining the `best' algorithm for the given problem, taking~\cref{eq:spatial-RMSE} and extending to all temporal realizations of the variables.
The algorithms are all run for a varying degree of particle number $N$ in the range of $100$ to $500$.

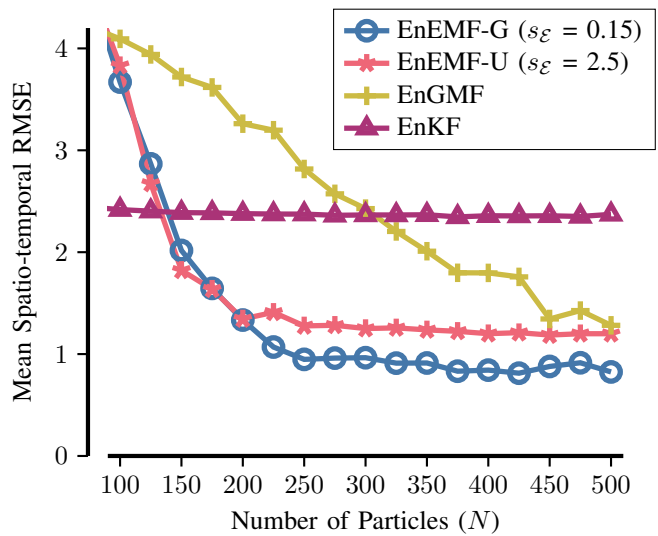
\begin{figure}[t]
    \centering
    \begin{tikzpicture}
    \begin{axis}[clean,
        cycle list name=tol,
        xtick={100, 150, 200, 250, 300, 350, 400, 450, 500},
        table/col sep=comma,
        xmin = 90,
        xmax = 510,
        ymin = 0,
        ymax = 4.2,
        clip = true,
        xlabel = {Number of Particles ($N$)},
        ylabel = {Mean Spatio-temporal RMSE},
        every axis plot/.append style={line width=2pt, mark size=3.5pt},
        legend style={at={(0.75,0.88)},anchor=center},
        legend cell align={left}]

    \addplot[mark=o,color=tolblue] table [x=Ns, y=rmseEnEMF, col sep=comma] {data/L96datajournal.csv};
    \addlegendentry{EnEMF-G ($s_{\mathcal{E}}$ = 0.15)};







    \addplot [color=tolred,mark=star, mark options={solid}] table [x=Ns, y=rmseEnEMFUKF, col sep=comma] {data/L96datajournal.csv};
    \addlegendentry{EnEMF-U ($s_{\mathcal{E}}$ = 2.5)};

    
    \addplot[color=tolyellow, mark=+] table [x=Ns, y=rmseEnGMF, col sep=comma] {data/L96datajournal.csv};
    \addlegendentry{EnGMF};

    \addplot [color=tolpurple,mark=triangle, mark options={solid}] table [x=Ns, y=rmseEnKF, col sep=comma] {data/L96datajournal.csv};
    \addlegendentry{EnKF};
    
    \end{axis}
    \end{tikzpicture}
    \caption{Number of particles ($N$) versus spatio-temporal RMSE for the Lorenz '96 problem for the EnEMF (both with $s_{\mathcal{E}} = 1$ and $s_{\mathcal{E}} = 1/2$), EnGMF, and the EnKF}
    \label{fig:lorenz96-experiment}
\end{figure}

The results of the experiment are provided in \cref{fig:lorenz96-experiment}.
As can be seen, the linearized EnKF, which is an (almost) linear filter is at an error plateau for all tested particle numbers. The EnEMF-U reaches a significantly lower error plateau for $N=200$ particles, and the EnEMF-G reaches the lowest error at $N=250$ particles. The EnGMF takes $N=500$ particles to barely reach the error of the EnEMG-U. More importantly, both EnEMF algorithms take $N=150$ particles to beat the EnKF while the EnGMF takes $N=325$, which is more than double. Given more optimal resampling,  a better Gaussian sum algorithm, and a better method to calculate the EnEMF weights it is possible that the EnEMF can reach an even smaller amount of required particles, though this is yet to be confirmed.

\section{Conclusions}
\label{sec:conclusions}

In this work we provide a theoretical derivation of the ensemble Epanechnikov mixture filter, and shown that it should outperform its cousin, the ensemble Gaussian mixture filter in the high-dimensional setting. 
We derive a practical implementation of the EnEMF that leverages the computational machinery of the EnGMF, such that the EnEMF can be implemented without significant computational overhead.

We additionally show through a numerical experiment that the EnEMF---while not attaining the theoretical error leaps over the EnGMF---still requires half the particles for the same level of error for a 40-variable problem.
These results provide a promising path forward for making use of the Epanechnikov kernel in high-dimensional particle filtering applications.

Future work will focus on better resampling techniques, a more optimal Gaussian sum update and of course, a more optimal weighting scheme.

\section*{Acknowledgment}
This work is based on research that is in part sponsored by the Air Force Office of Scientific Research (AFOSR) under grant FA9550-23-1-0646

\bibliographystyle{IEEEtran}
\bibliography{bibsmall, bibfiles/banana,bibfiles/covarianceshrinkage,bibfiles/em,bibfiles/engmf,bibfiles/filteringgeneral,bibfiles/kernelapproximation,bibfiles/misc,bibfiles/multifidelity, bibfiles/enkf}

\begin{IEEEbiography}[{\includegraphics[width=1in,height=1.25in,clip,keepaspectratio]{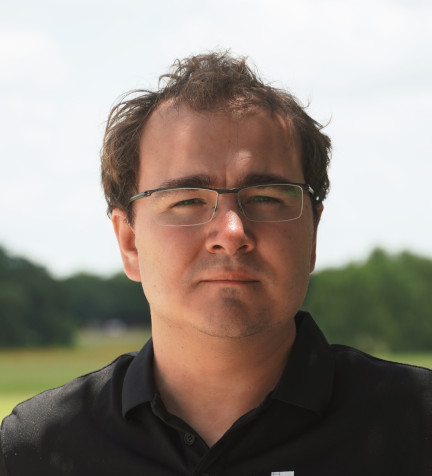}}]{Andrey A. Popov} received his B.S. in Mathematics from RPI, Troy, NY, and his Ph.D. in Computer Science from Virginia Tech, Blacksburg, VA. He is currently a postdoc at the Oden Institute at the University of Texas at Austin, Austin, TX.
His research interests include state estimation/data assimilation, theory-guided machine learning, and data-driven methods for both modeling and quantifying uncertainty in dynamical systems.
\end{IEEEbiography}

\begin{IEEEbiography}[{\includegraphics[width=1in,height=1.25in,clip,keepaspectratio]{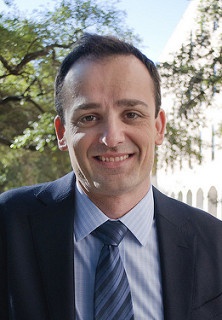}}]{Renato Zanetti} (Senior Member, IEEE) is an Associate Professor of Aerospace Engineering at The University of Texas at Austin. Prior to joining UT he worked as an engineer at the NASA Johnson Space Center and at Draper Laboratory. Renato’s research interests include nonlinear estimation, onboard navigation, and autonomous aerospace vehicles.

\end{IEEEbiography}

\end{document}